%% file: main.tex
\documentclass[letterpaper,10pt,journal,twoside]{IEEEtran}

\IEEEoverridecommandlockouts %

\RequirePackage[loading]{tracefnt}
\input{macros.tex}

\newcommand{\rev}[1]{#1}

\title{GraffMatch: Global Matching of 3D Lines and Planes for Wide Baseline LiDAR Registration}

\author{Parker C. Lusk$^{1}$,  Devarth Parikh$^{2}$, Jonathan P. How$^{1}$%
    \thanks{Manuscript received: August, 16, 2022; Revised November, 5, 2022; Accepted November, 28, 2022.}%
    \thanks{This paper was recommended for publication by Editor Sven Behnke upon evaluation of the Associate Editor and Reviewers' comments.
This work was supported by the Ford Motor Company.} %
\thanks{$^{1}$Department of Aeronautics and Astronautics, Massachusetts Institute of Technology.
        {\tt\footnotesize \{plusk, jhow\}@mit.edu}}%
\thanks{$^{2} $Ford Motor Company.
        {\tt\footnotesize dparikh9@ford.com}}%
\thanks{Digital Object Identifier (DOI): see top of this page.}
}

\IEEEaftertitletext{\vspace{-1\baselineskip}}

\begin{document}

\maketitle

\markboth{IEEE Robotics and Automation Letters. Preprint Version. Accepted November, 2022}
{Lusk \MakeLowercase{\textit{et al.}}: GraffMatch: Global Matching of 3D Lines and Planes for Wide Baseline LiDAR Registration} 

\begin{abstract} 
\input{paper/abstract}
\end{abstract}


\input{paper/intro}

\input{paper/related_work}

\input{paper/method}

\input{paper/experiments}

\input{paper/conclusion}
\input{paper/appendix}

\balance %

\bibliographystyle{IEEEtran}
\bibliography{refs}

\end{document}

%% file: macros.tex
\usepackage{amsmath} %
\usepackage{amssymb} %
\usepackage{amsfonts}
\usepackage[sort,compress]{cite}
\usepackage[usenames,dvipsnames]{xcolor}
\usepackage{scalerel}
\usepackage{graphicx}
\usepackage{tabularx}
\usepackage{multirow}
\usepackage{mathtools}
\usepackage{dsfont}
\usepackage{esvect} %
\usepackage[binary-units=true]{siunitx}
\usepackage{caption}
\usepackage[pdftex, pdfstartview={FitV}, pdfpagelayout={TwoColumnLeft},bookmarksopen=true,plainpages = false, colorlinks=true, linkcolor=black, citecolor = black, urlcolor = black,filecolor=black , pagebackref=false,hypertexnames=false, plainpages=false, pdfpagelabels ]{hyperref}
\usepackage[T1]{fontenc} %
\usepackage{mathtools, cuted} %
\usepackage{bm}
\usepackage{bbm}
\usepackage{xspace}
\usepackage{stackengine}

\usepackage{balance} %
\usepackage{booktabs} %
\usepackage[export]{adjustbox} %

\newcolumntype{x}[1]{%
>{\raggedleft\hspace{0pt}}p{#1}}%

\definecolor{scan1blue}{HTML}{0072bd}
\definecolor{scan2red}{HTML}{d95319}

\usepackage{algorithm}
\usepackage[noend]{algpseudocode} %
\definecolor{commentclr}{RGB}{34, 139, 34}

\usepackage{tikz}
\usetikzlibrary{plotmarks}

\newcommand{\tikzcircle}[2][red,fill=red]{\tikz[baseline=-0.5ex]\draw[#1,radius=#2] (0,0) circle ;}%

\usepackage{subcaption}
\DeclareCaptionLabelSeparator{periodspace}{.\quad}
\newcommand{\ra}[1]{\renewcommand{\arraystretch}{#1}}

\usepackage{amsthm}

\newtheorem{prop}{Proposition}

\theoremstyle{definition}

\newtheorem{definition}{Definition}

\makeatletter
\newcommand{\newreptheorem}[2]{%
\newtheorem*{rep@#1}{\rep@title}%
\newenvironment{rep#1}[1]{%
 \def\rep@title{#2 \ref*{##1}}%
 \begin{rep@#1}}%
 {\end{rep@#1}}}
\makeatother
\newreptheorem{prop}{Proposition}

\usepackage{letltxmacro}
\LetLtxMacro\orgvdots\vdots
\LetLtxMacro\orgddots\ddots

\makeatletter
\DeclareRobustCommand\vdots{%
	\mathpalette\@vdots{}%
}
\newcommand*{\@vdots}[2]{%
	\sbox0{$#1\cdotp\cdotp\cdotp\m@th$}%
	\sbox2{$#1.\m@th$}%
	\vbox{%
		\dimen@=\wd0 %
		\advance\dimen@ -3\ht2 %
		\kern.5\dimen@
		\dimen@=\wd2 %
		\advance\dimen@ -\ht2 %
		\dimen2=\wd0 %
		\advance\dimen2 -\dimen@
		\vbox to \dimen2{%
			\offinterlineskip
			\copy2 \vfill\copy2 \vfill\copy2 %
		}%
	}%
}
\DeclareRobustCommand\ddots{%
	\mathinner{%
		\mathpalette\@ddots{}%
		\mkern\thinmuskip
	}%
}
\newcommand*{\@ddots}[2]{%
	\sbox0{$#1\cdotp\cdotp\cdotp\m@th$}%
	\sbox2{$#1.\m@th$}%
	\vbox{%
		\dimen@=\wd0 %
		\advance\dimen@ -3\ht2 %
		\kern.5\dimen@
		\dimen@=\wd2 %
		\advance\dimen@ -\ht2 %
		\dimen2=\wd0 %
		\advance\dimen2 -\dimen@
		\vbox to \dimen2{%
			\offinterlineskip
			\hbox{$#1\mathpunct{.}\m@th$}%
			\vfill
			\hbox{$#1\mathpunct{\kern\wd2}\mathpunct{.}\m@th$}%
			\vfill
			\hbox{$#1\mathpunct{\kern\wd2}\mathpunct{\kern\wd2}\mathpunct{.}\m@th$}%
		}%
	}%
}
\makeatother

\let\oldnl\nl%
\newcommand{\nonl}{\renewcommand{\nl}{\let\nl\oldnl}}%
\makeatother

\def\sl{\mathcal{L}}

\def\Gr{\mathrm{Gr}}
\def\Graff{\mathrm{Graff}}
\def\bA{\mathbb{A}}
\def\bY{\mathbb{Y}}
\def\S{\mathcal{S}}

\def\recall{\tiny \stackanchor{Recall}{[\%]}}
\def\auc{\tiny \stackanchor{LMR}{AUC}}
\def\rerr{\tiny \stackanchor{$R_\text{e}$}{[deg]}}
\def\terr{\tiny \stackanchor{$t_\text{e}$}{[cm]}}
\def\time{\tiny \stackanchor{$t$}{[ms]}}

\newcommand{\M}[1]{{\bm #1}} %
\renewcommand{\boldsymbol}[1]{{\bm #1}}

\newcommand{\MR}{\M{R}}

\newcommand{\vt}{\boldsymbol{t}}

\newcommand\eqdef{\mathrel{\overset{\makebox[0pt]{\mbox{\normalfont\tiny def}}}{=}}}

\graphicspath{{figures/}}

\usepackage{svg}
\svgpath{{figures/}}

%% file: paper/abstract.tex
Using geometric landmarks like lines and planes can increase navigation accuracy and decrease map storage requirements compared to commonly-used LiDAR point cloud maps.
However, landmark-based registration for applications like loop closure detection is challenging because a reliable initial guess is not available.
Global landmark matching has been investigated in the literature, but these methods typically use ad hoc representations of 3D line and plane landmarks that are not invariant to large viewpoint changes, resulting in incorrect matches and high registration error.
To address this issue, we adopt the affine Grassmannian manifold to represent 3D lines and planes and prove that the distance between two landmarks is invariant to rotation and translation if a shift operation is performed before applying the Grassmannian metric.
This invariance property enables the use of our graph-based data association framework for identifying landmark matches that can subsequently be used for registration in the least-squares sense.
Evaluated on a challenging landmark matching and registration task using publicly-available LiDAR datasets, \rev{our approach yields a $1.7${\normalfont x} and $3.5${\normalfont x} improvement} in successful registrations compared to methods that use viewpoint-dependent centroid and ``closest point'' representations, respectively.

%% file: paper/intro.tex
\section{Introduction}\label{sec:intro}

\IEEEPARstart{E}{stimating} the rigid-body transformation between two sensors is a fundamental component of many mobile robotic systems.
Wide-baseline registration is particularly challenging since an odometry signal may not be accurate, or even available, to use as an initial guess.
This situation arises in core tasks such as loop closure generation, multi-robot map merging, extrinsic calibration, and global (re)localization.
In these cases, the relative rotation and translation between sensors can instead be accurately estimated by matching co-visible features and optimizing for the best feature alignment.

In visual settings, appearance-based descriptors are commonly used for image retrieval or place recognition~\cite{lowry2015visual,garg2021where,galvez2012bags}.
However, appearance-based techniques are often limited due to their sensitivity to illumination, weather, and viewpoint changes.
Working with 3D sensors can alleviate these issues because of the geometric nature of the data~\cite{kim2021scan,zeng20173dmatch}, but this requires storing and processing large point clouds, which can hinder online operation.
In the context of LiDAR-based navigation, using geometric landmarks such as lines and planes has resulted in storage-efficient maps~\cite{schaefer2019long,kummerle2019accurate,cao2020accurate}, better scene understanding~\cite{cao2021lidar,bavle2022situational}, and low-drift odometry and mapping~\cite{zhang2014loam,kaess2015simultaneous,geneva2018lips,zhou2022plc}.
However, existing geometric landmark matching techniques tend to assume 2D motion only~\cite{cao2020accurate,cao2021lidar}, use local association strategies given an initial guess from odometry~\cite{zhou2021pi,zhou2022plc}, or use a series of heuristic checks~\cite{pathak2010fast,fernandez2013fast,jiang2020lipmatch} that lead to low matching success rate.
\rev{SegMatch~\cite{dube2017segmatch} uses point clusters as landmarks, but the repeatability of cluster segmentation may strongly depend on viewing angle, especially in the presence of partial occlusion or object motion.}

Recent successes in graph-based data association~\cite{lusk2021clipper,shi2021robin} have significantly increased the robustness of the correspondence selection process in spatial perception problems by leveraging the notion of pairwise consistency.
By matching constellations of objects that have consistent pairwise distances across two views, good correspondences can be identified even in the presence of many bad ones.
Global data association techniques do not require an initial registration guess, but rely on a correctly defined pairwise distance that is invariant to transformation; for example, the Euclidean distance between two rigidly-attached points does not change even as those points are rotated and translated.
Existing works frequently use Euclidean distance between line or plane pairs, using representations such as the centroid~\cite{cao2020accurate}, which is not well-defined for infinite lines and planes, or the ``closest point'' (CP) parameterization~\cite{geneva2018lips}, which lacks the necessary translation invariance because of its dependence on sensor origin.

\begin{figure}[t]
    \centering
    \vstretch{1}{\includegraphics[trim=1cm 4cm 8cm 20cm, clip, width=\columnwidth]{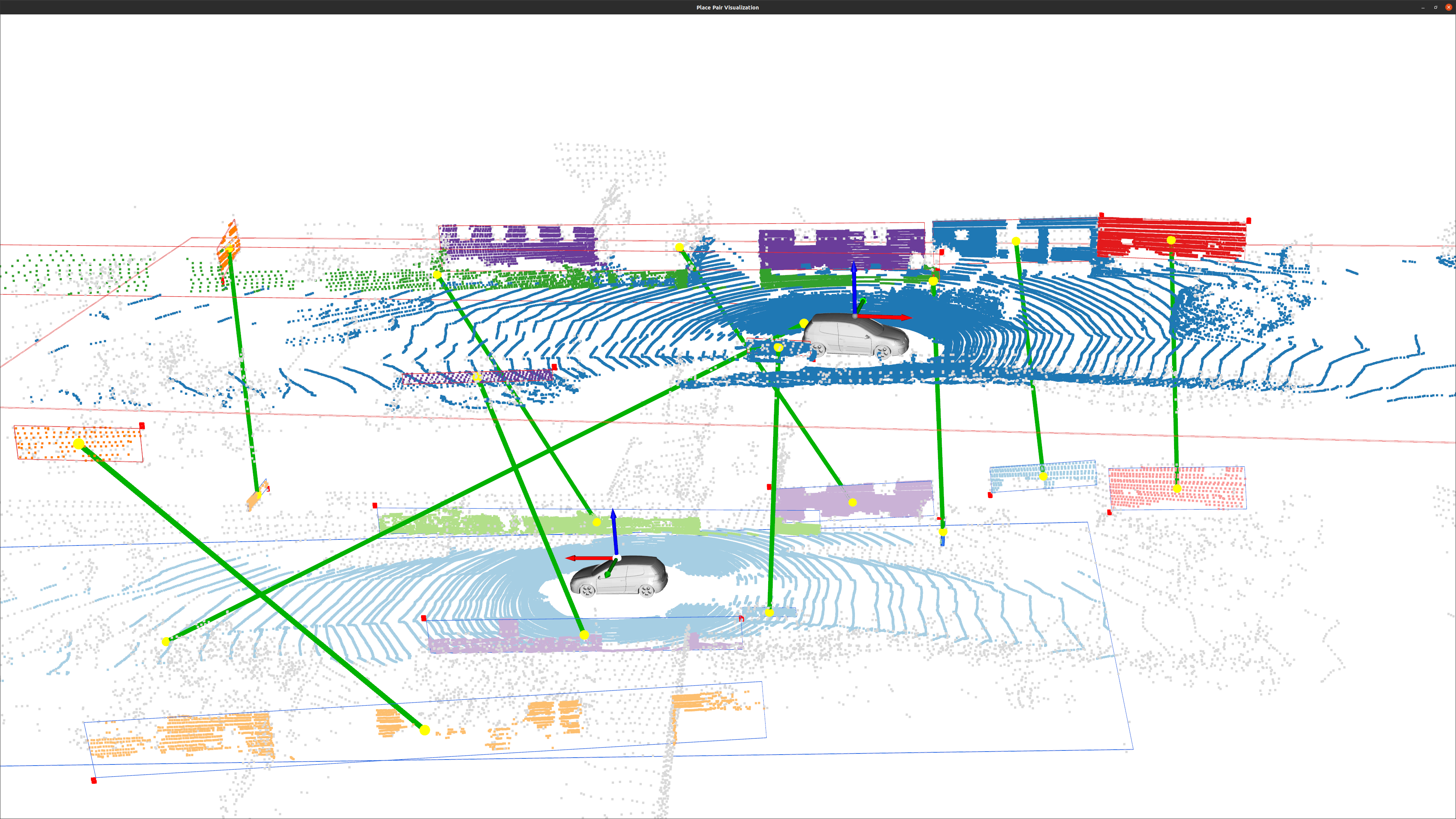}}
    \caption{Successful matching and alignment of line and plane landmarks from two LiDAR scans \SI{180}{\deg} and \SI{14}{\m} apart, without an initial guess.
    Sensor origins are denoted by the coordinate axes on top of the cars and scans are offset in the $z$ direction for visualization.
    Poles and planar patches extracted from each LiDAR scan are represented as 3D affine Grassmannian elements.
    Using the associated Riemannian metric allows for the evaluation of geometric consistency between landmark pairs.
    Correspondences (shown with green lines connecting landmark centroids) are identified using our GraffMatch algorithm and then used to estimate the rotation and translation between the two sensors, yielding an alignment error of \SI{0.8}{\deg} and \SI{12}{\cm}.
    }
    \label{fig:teaser-image}
\end{figure}

Instead, we represent line and plane landmarks naturally as elements of a Grassmannian manifold, which is the space of all linear subspaces.
In particular, we utilize the \emph{affine} Grassmannian manifold, which allows for the representation of affine subspaces (i.e., linear subspaces not necessarily containing the origin).
Thus, invariant distances between geometric primitives can easily be defined in a principled manner using the Grassmannian metric, enabling the use of our robust, global data association framework~\cite{lusk2021clipper}.
Then, the rigid transformation between a pair of candidate loop closure scans can be estimated by solving a line and plane registration problem with known correspondences in the least-squares sense.

This letter presents an evolution of our previous work~\cite{lusk2022global} by providing further validation on additional datasets, an in-depth experimental analysis, and new applications of our method in automatic extrinsic sensor calibration.
In summary, our main contributions are:
\begin{itemize}
    \item the introduction of the affine Grassmannian for global data association of lines and planes, leading to 3D registration without requiring an initial guess;
    \item a least squares estimator for rigid transformation \rev{using line and plane representations directly, leading to accurate rotation and translation estimation;}
    \item experimental evaluation in the context of loop closure, using challenging scan pairs of LiDAR datasets, showing superior recall and accuracy over the state-of-the-art.
\end{itemize}
We emphasize that this is the first work using the affine Grassmannian manifold for data association, which provides a unifying and principled framework for associating points, lines, planes (or higher dimensional linear objects) in spatial and geometric perception problems encountered in robotics. %

%% file: paper/related_work.tex
\section{Related Work}\label{sec:relatedwork}

Geometric registration algorithms are either local or global, depending on if an initial guess is required.
Local methods are often used in scan matching (e.g., ICP~\cite{besl1992method}), where consecutive scans typically have small displacement between them.
In contrast, global methods do not need an initial guess to succeed and are often preferred in settings like loop closure detection because no good initial guess exists.
In point-based registration, global methods first generate candidate point correspondences, typically based on local descriptors~\cite{rusu2009fast,zeng20173dmatch,choy2019fully}.
The set of putative correspondences are likely to contain incorrect matches, called outliers, and so robust iterative estimation techniques like RANSAC~\cite{fischler1981random} or graduated non-convexity~\cite{zhou2016fast,yang2020teaser} can be used to select a subset of correspondences that best support the model.
Recently, non-iterative robust approaches based on graph theory~\cite{lusk2021clipper,shi2021robin} have been introduced that significantly improve the performance of estimation in the presence of outliers.
In this work, we apply our CLIPPER algorithm~\cite{lusk2021clipper} to the domain of global matching of 3D lines and planes.

Poles/lines and planes commonly exist in man-made environments and have recently been used as landmarks in LiDAR navigation.
The benefits of using higher-order geometric primitives include lower storage and processing requirements since there are fewer pole and plane objects than points~\cite{schaefer2019long,kummerle2019accurate,cao2020accurate}, and more accurate odometry because it is infeasible to get exact point-to-point correspondences from sparse LiDAR point clouds~\cite{zhou2021pi}.
Local methods for landmark-based registration rely on identifying the closest landmarks between two scans, given an initial alignment guess.
Nearest neighbor search requires calculating distances between landmarks, which is dependent on the landmark representation.
The most common representations include vector form for lines, Hesse normal form for planes, and the CP vector for lines/planes~\cite{yang2019observability}.
The CP representation compactly encodes position and vector orientation in a 3-vector, and the Euclidean distance is often used to find similar landmarks~\cite{geneva2018lips,zhou2021pi}.
For lines in vector form or planes in Hesse normal form, both the angle between vectors and the distance between points is used~\cite{kaess2015simultaneous}.
If points are retained in one of the scans, then the point-to-landmark RMSE can also be used~\cite{zhou2022plc}.
Other methods~\cite{schaefer2019long,kummerle2019accurate,wilbers2019localization} project the centroids of the detected landmark onto the ground plane, creating 2D points.
However, these methods assume the landmark (infinite line or plane) has a well-defined centroid, the ground plane is known, and that 2D registration is sufficient.

Since descriptors for 3D lines and planes have not been thoroughly explored, most global methods rely on a series of geometric tests to assign correspondences~\cite{pathak2010fast}.
LiPMatch~\cite{jiang2020lipmatch} adopts an interpretation tree for plane matching~\cite{fernandez2013fast} in loop closure detection, using unary and binary constraints between candidate plane matches to determine the largest set of consistent matches.
However, some of these constraints are sensitive to viewpoint change (e.g., centroid, area).
ClusterMatch~\cite{cao2021lidar,cao2020accurate} matches poles/lines and planes by iteratively searching for landmark pairs with similar pairwise centroid distance until a large number of matches supports the resulting transformation.
A critical drawback of these methods is that, although global methods, the heuristic geometric tests that are used are often heavily view dependent.
In contrast, our method performs global registration by searching for a set of correspondences that have consistent \rev{intrascan} pairwise distances, using a distance definition based on their natural representation as elements of Grassmannian manifolds.

The Grassmannian has been used extensively in subspace learning~\cite{hamm2008grassmann}, especially in face recognition~\cite{huang2015projection} and appearance tracking~\cite{shirazi2014object} tasks in computer vision.
Calinon~\cite{calinon2020gaussians} outlines the use of Riemannian manifolds in robotics and notes the under-representation of the Grassmannian manifold.

\subsection{Preliminaries}\label{sec:preliminaries}

We briefly introduce the Grassmannian manifold, but a more comprehensive introduction is provided in~\cite{edelman1998geometry}.
The Grassmannian is the space of $k$-dimensional subspaces of $\mathbb{R}^n$, denoted $\Gr(k,n)$.
For example, $\Gr(1,3)$ represent 3D lines containing the origin.
An element $\bA\in\Gr(k,n)$ is represented by an orthonormal matrix $A\in\mathbb{R}^{n\times k}$ whose columns form an orthonormal basis of $\bA$.
Note that the choice of $A$ is not unique.
The geodesic distance between two subspaces $\bA_1\in\Gr(k_1,n)$ and $\bA_2\in\Gr(k_2,n)$ is
\begin{equation}
d_\mathrm{Gr}(\bA_1, \bA_2) = \left(\sum_{i=1}^{\min(k_1,k_2)} \theta_i^2\right)^{1/2}
\end{equation}
where $\theta_i$ are known as the principal angles~\cite{bjorck1973numerical}.
These angles can be computed via the singular value decomposition (SVD) of the corresponding orthonormal matrices of $\bA_1$ and $\bA_2$,
\begin{equation}
A_1^\top A_2 = U\, \mathrm{diag}(\cos\theta_1, \dots, \cos\theta_{\min(k_1,k_2)} )\, V^\top.
\end{equation}
\vskip-0.0em
\begin{figure}[t]
    \centering
    \begin{subfigure}[b]{0.49\columnwidth}
        \includeinkscape[pretex=\footnotesize,width=\columnwidth]{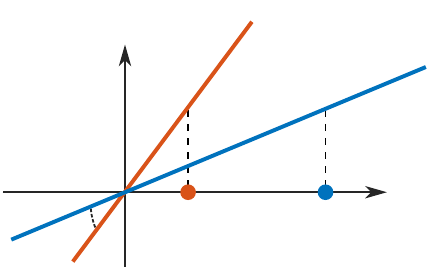}
        \caption{}
        \label{fig:graffexample}
    \end{subfigure}
    \begin{subfigure}[b]{0.49\columnwidth}
        \includegraphics[trim = 0mm 0mm 0mm 0mm, clip, width=1\linewidth]{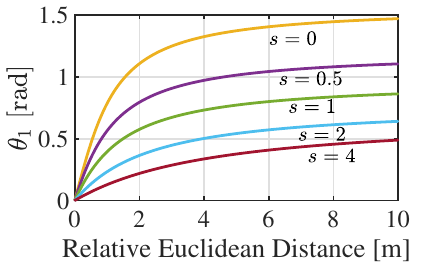}
        \caption{}
        \label{fig:graffsensitivity}
    \end{subfigure}
    \caption{(a) Example of points in $\Graff(0,1)$ being embedded as lines in $\Gr(1,2)$.
    The principal angle between these two linear subspaces is $\theta_1$.
    (b) When applied directly, $d_\Graff$ is not invariant to global translation $s$.
    }
    \label{fig:graffexample-both}
\end{figure}

We are specifically interested in affine subspaces of $\mathbb{R}^3$, e.g., lines and planes that may be at some distance away from the origin.
In analogy to $\mathrm{Gr}(k,n)$, the set of $k$-dimensional affine subspaces constitute a smooth manifold called the \emph{affine Grassmannian} and denoted $\Graff(k,n)$~\cite{lim2021grassmannian}.
We write an element of this manifold as $\bY=\bA+b\in\Graff(k,n)$ with affine coordinates $[A,b]\in\mathbb{R}^{n\times(k+1)}$, where $A\in\mathbb{R}^{n\times k}$ is an orthonormal matrix and $b\in\mathbb{R}^n$ is the displacement of $\bA$ from the origin.
We emphasize that $\Graff(k,n)\neq\Gr(k,n)\times\mathbb{R}^n$.
Instead, an element $\bY\in\Graff(k,n)$ is treated as a higher-order subspace via the embedding
\begin{align}
j:\Graff(k,n)&\hookrightarrow\Gr(k+1,n+1), \notag \\ 
\bA+b &\mapsto \mathrm{span}(\bA\cup\{b+e_{n+1}\}),
\end{align}
where $e_{n+1} = (0,\dots,0,1)^\top\in\mathbb{R}^{n+1}$ (see \cite[Theorem 1]{lim2021grassmannian}).
Fig.~\ref{fig:graffexample} shows an example of two points $\bY_1$ and $\bY_2$ in $\mathbb{R}$ being embedded as different lines $j(\bY_1)$ and $j(\bY_2)$ in $\mathbb{R}^2$.

The Stiefel coordinates of $\bY\in\Graff(k,n)$,
\begin{equation}
Y =
\begin{bmatrix}
A & b_0/\sqrt{1+\|b_0\|^2} \\
0 & 1/\sqrt{1+\|b_0\|^2}
\end{bmatrix}\in\mathbb{R}^{(n+1)\times(k+1)},
\end{equation}
allow for the computation of distances between two affine subspaces using the Grassmannian metric,
\begin{equation}
d_\Graff(\bY_1,\bY_2) = d_\Gr(j(\bY_1),j(\bY_2)),
\end{equation}
with principal angles computed via the SVD of $Y_1^\top Y_2$.
The vector $b_0\in\mathbb{R}^n$ is the orthogonal displacement of $\bA$, which is the projection of $b$ onto the left nullspace of $A$ s.t. $A^\top b_0=0$.

For convenience, the line $\bY^\ell\in\Graff(1,3)$ may also be represented in vector form as $\ell = [A;b]\in\mathbb{R}^6$, and a plane $\bY^\pi\in\Graff(2,3)$ may be represented in Hesse normal form as $\pi = [n;d]\in\mathbb{R}^4$ where $n = \mathrm{ker}\,A^\top$ and $d = \|b_0\|$.
Under a rigid transformation $T=(R,t)\in\mathrm{SE}(3)$, the transformation law of lines and planes can be written
\begin{align}
\ell' &= f_\ell(\ell,R,t) := \begin{bmatrix}RA&Rb+t\end{bmatrix}^\top \\
\pi'  &= f_\pi(\pi,R,t) := T^{-\top}\pi.
\end{align}

%% file: paper/method.tex
\section{Method}\label{sec:method}

Given a set ${\S_i = \{ \bY^\ell_1,\dots,\bY^\ell_{l_i}, \bY^\pi_1,\dots,\bY^\pi_{p_i}\}}$ with $l_i$ lines and $p_i$ planes, we refer to the \rev{$a$-th landmark} as $s_{i,a}\in\S_i$.
Our method is comprised of the following steps: (i) constructing a consistency graph based on pairwise landmark distances, (ii) identifying landmark correspondences via the densest complete subgraph in the consistency graph, and (iii) estimating a rigid transformation based on the identified correspondences.

\begin{figure}[t]
    \centering
    \includeinkscape[pretex=\footnotesize,width=\columnwidth]{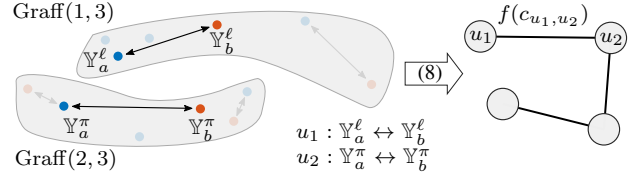}
    \caption{
    Construction of a consistency graph.
    Using $d_\Graff$, the distance between a line and a plane in set $\S_i$ (\tikzcircle[scan1blue,fill=scan1blue]{1.5pt}) is compared to the distance between the two corresponding landmarks in set $\S_j$ (\tikzcircle[scan2red,fill=scan2red]{1.5pt}).
    The consistency of these two distances is evaluated using \eqref{eq:consistency} and the edge $(u_1,u_2)$ is so weighted.
    }
    \label{fig:consistencygraph}
\end{figure}

\subsection{Consistency Graph Construction}\label{sec:consistency}

\rev{
A consistency graph for two sets $\mathcal{S}_i$, $\mathcal{S}_j$ is an undirected weighted graph $\mathcal{G}=(\mathcal{V},\mathcal{E},w)$.
Each vertex $u\in\mathcal{V}$ represents a potential correspondence between landmarks of the same type, denoted $s_{i,a}\leftrightarrow s_{j,b}$.
An edge $(u_p,u_q)\in\mathcal{E}$ between vertices $u_p,u_q\in\mathcal{V}$ exists if the correspondence pair is \emph{consistent}; the edge-weighting function $w:\mathcal{E}\to[0,1]$ indicates the level of consistency.
The correspondence pair $u_p,u_q$ is consistent if the distance between the underlying landmarks satisfies %
\begin{equation}\label{eq:consistency}
c_{u_p,u_q} \eqdef |d(s_{i,u_p^i},\,s_{i,u_q^i}) - d(s_{j,u_p^j},\,s_{j,u_q^j})| < \epsilon,
\end{equation}
where $d$ is some distance function and, by some abuse of notation, $u_p^i$ denotes the $a$-th landmark of $\mathcal{S}_i$ involved in the correspondence $u_p$.
Note that the two distances in \eqref{eq:consistency} are between landmarks \emph{internal} to sets $\S_i$ and $\S_j$, respectively.
If a pair of correspondences are deemed consistent, the corresponding edge is attributed the weight $w(u_p,u_q)\eqdef f(c_{u_p,u_q})$, for some choice of ${f:\mathbb{R}_+\to[0,1]}$ that scores very consistent pairs close to 1.
In this paper, we choose $f(c)\eqdef\exp(-c^2/2\sigma^2)$ for simplicity, although other appropriate functions could be used.
Given a consistency graph, correspondences are selected that maximize consistency, further explained in Section~\ref{sec:clipper}.}

The distance function $d$ must be carefully chosen to ensure accuracy of graph-based data association.
In particular, we desire \eqref{eq:consistency} to hold when $s_{j,u_1^b},\,s_{j,u_2^b}$ are the transformed versions of $s_{i,u_1^a},\,s_{i,u_2^a}$, respectively.
This invariance property leads to subgraphs of the consistency graph that indicate a set of landmark matches.
\begin{definition}\label{defn:invariance}
A distance $d:X\times X\to\mathbb{R}$ is \emph{invariant} if $d(x_1,x_2) = d(x_1',x_2')$, where $x_1',x_2'\in X$ are the transformation of $x_1,x_2\in X$ under $T\in\mathrm{SE}(3)$, respectively.
\end{definition}
We establish the invariance of the metric $d_\Graff$ to rotation and, under careful application, translation.
\begin{prop}\label{prop:invariance}
\input{paper/prop_invariance}
\end{prop}
\begin{proof}
See Appendix A.
\end{proof}
The intuition of Proposition~\ref{prop:invariance} can be understood from Fig.~\ref{fig:graffexample-both}.
As $\bY_1$ and $\bY_2$ are together translated further from the origin, the principal angle between $j(\bY_1)$ and $j(\bY_2)$ decreases to zero in the limit.
However, the distance between the affine components of $\bY_1$ and $\bY_2$ remains the same, no matter the translation.
By first shifting the affine components, we remove the dependence of the absolute translation in the computation of the principal angle, while maintaining the dependence on the \emph{relative} translation between $\bY_1$ and $\bY_2$.

A remaining challenge is to address the insensitivity of $d_\Graff$ to the Euclidean distance between landmarks' affine components.
The yellow curve ($s=0$) in Fig.~\ref{fig:graffsensitivity} represents the principal angle between $\bY_1,\bY_2\in\Graff(0,1)$ after shifting them as per Proposition~\ref{prop:invariance}, as a function of the relative translation between $\bY_1$ and $\bY_2$.
Observe that after a distance of approximately \SI{2}{\meter}, the curve quickly asymptotes towards $\tfrac{\pi}{2}$.
This nonlinearity leads to poor discrimination between pairs of correspondences whose internal landmarks are far apart in the Euclidean sense.
To combat this when calculating pairwise affine Grassmannian distances, we first scale the affine component of each $\bY_i$ by a constant parameter $\rho$ so that the affine coordinates of $\bY_i$ become $[A_i,b_i/\rho]$.
The choice of $\rho$ depends on the average Euclidean distance between landmarks in the environment and its effect is to bring principal angles into the linear regime.
The selection of $\rho$ is discussed further in Section~\ref{sec:exp-scaling}.

With Proposition~\ref{prop:invariance} and the scaling parameter $\rho$ in hand, a consistency graph between landmarks in $\S_i$ and $\S_j$ can be constructed.
We establish initial correspondences between each landmark in $\S_i$ with each landmark of $\S_j$ so long as the landmarks are of the same dimensions $k$ (i.e., we do not allow lines to be associated to planes).
Given additional information such as color, scan intensity, planar patch area, or pole radius, this initial set of correspondences could be refined, but would rely on accurately segmenting lines and planes across potentially wide baselines.
While we restrict landmark correspondences to be of the same dimension, the machinery we have developed allows for computing the consistency of two correspondences whose internal pair of objects have differing dimension, thereby aiding in subgraph selection.
Fig.~\ref{fig:consistencygraph} illustrates evaluating the consistency of a correspondence pair using the affine Grassmannian.

\subsection{Graph-based Global Data Association}\label{sec:clipper}

Given a consistency graph, the task of matching objects from two scans is reduced to identifying the densest clique of consistent correspondences, formalized as the problem
\begin{gather}\label{eq:densestclique}
\begin{array}{ll}
\underset{u \in \{0,1\}^m}{\text{maximize}} & \dfrac{u^\top  M \, u}{u^\top u}
\\
\text{subject to} & u_i \, u_j = 0  \quad \text{if}~ M(i,j)=0, ~ \forall {i,j},
\end{array}
\end{gather}
where $M\in[0,1]^{m\times m}$ is the weighted adjacency matrix (i.e., from $w$ as defined in Section~\ref{sec:consistency}) with ones on the diagonal, and ${u\in\{0,1\}^m}$ indicates a consistent set of correspondences.
Note that we choose to maximize the \emph{density} of correspondences rather than the cardinality (i.e., maximum clique) as our previous work has found this objective to produce more accurate results~\cite{lusk2021clipper}.
Problem~\eqref{eq:densestclique} is NP-hard, therefore we solve a particular relaxation which yields high accuracy solutions
via our efficient CLIPPER algorithm
(see \cite{lusk2021clipper} for more details).

\subsection{Transformation Estimation}

Given pairwise correspondences between landmarks in $\S_i$ and $\S_j$, consider finding the best rigid transformation to simultaneously align matched lines and planes by solving the optimization problem
\begin{equation}\label{eq:reg} %
\min_{\substack{R\in\mathrm{SO}(3),\\ t\in\mathbb{R}^3}}
\sum_{i=1}^{p} \|\pi_i' - f_\pi(\pi_i,R,t)\|^2
+
\sum_{i=1}^{l} \|\ell_i' - f_\ell(\ell_i,R,t)\|^2.
\end{equation}
This problem can be solved in closed-form by first solving for the rotation via SVD, then solving for the translation via least squares, similar to Arun's method for point cloud registration~\cite{arun1987least}.
\rev{Note that in the case of only parallel planes or only collinear lines, \eqref{eq:reg} will be degenerate (i.e., the correlation matrix used in SVD will be rank deficient).
These cases can be numerically detected via the condition number $\kappa$ of the correlation matrix.}
The benefit of using the line and plane geometry directly, as opposed to a point parameterization, is twofold.
First, it allows the use of the full information present in the infinite plane or line, i.e., distance from origin as well as orientation.
Second, it does not require assumptions about where the ``centroid'' of the plane or line is, which is undefined for infinite planes and lines and requires consistent segmentation of landmarks from point clouds.
Together, these benefits lead to a more accurate rigid transformation estimate when aligning line and plane landmarks.

%% file: paper/prop_invariance.tex
For elements $\bY_1\in\Graff(k_1,3)$, $\bY_2\in\Graff(k_2,3)$ with affine coordinates $[A_1,b_1]$ and $[A_2,b_2]$, the affine Grassmannian metric $d_\Graff$ is invariant if the affine components are first \emph{shifted} to the origin, i.e., if both $b_1$ and $b_2$ are first translated by $-b_1$.

%% file: paper/experiments.tex
\section{Experiments}\label{sec:experiments}

We evaluate our method, called GraffMatch, using LiDAR scans from three datasets: KITTI~\cite{geiger2012kitti} sequences 00, 02, 05, 08; KITTI-360~\cite{liao2022kitti360} sequences 00, 04, 06, 09; NCLT~\cite{carlevaris2016nclt} sessions 2012-04-29, 2012-05-11, 2012-12-01.
We include comparisons with CPMatch~\cite{zhou2021pi}, BruteForceRMSE~\cite{zhou2022plc}, ClusterMatch~\cite{cao2020accurate}, and LiPMatch~\cite{jiang2020lipmatch}, all of which are used for geometric landmark registration in state-of-the-art pipelines.
CPMatch uses nearest neighbor search on CP vectors and BruteForceRMSE exhaustively calculates the point-to-line/plane RMSE between each potential landmark match, followed by the selection of correspondences with low RMSE.
ClusterMatch creates putative correspondences using pairwise landmark centroid distances followed by inlier validation, and LiPMatch uses an interpretation tree to find plane correspondences with similar geometric properties.
We also include comparisons using our CLIPPER algorithm, using Euclidean distance of centroids and CP vectors to score consistency, both of which lack the required invariance for lines and planes.
The algorithms are implemented using Python and C++ and executed on an i9-7920X CPU with 64 GB RAM.
The parameters used for GraffMatch (see \eqref{eq:consistency}) are $\epsilon=0.2$ and $\sigma=0.05$.
In our comparisons, we show that GraffMatch successfully matches more landmarks, leading to more successful registrations.

\subsection{Dataset Preparation}\label{sec:dataset}

Motivated by place recognition and loop closure detection applications, we sample poses along the ground-truth trajectory of each dataset sequence with a stride of \SI{2}{\m} to create \emph{places} (i.e., a pose with the LiDAR scan at that pose), following~\cite{kim2021scan}.
For each place $p$, true \emph{revisits} are found by identifying any previously visited place $p'$ such that the path length between $p$ and $p'$ is greater than $\SI{50}{\m}$ and the Euclidean distance between the places is less than $\SI{16}{\m}$.
Ground-truth landmark matches between revisited places are generated by aligning the landmarks using ground truth and solving a linear-sum assignment problem based on $d_\mathrm{Graff}$ distances.
Valid place pairs are selected as \emph{loop closures} from revisits having more than $4$ true landmark matches \rev{with registration conditioning $\kappa<1e3$} and having registration error less than \SI{5}{\deg} and \SI{1}{\m} using the true landmark matches.
This setup provides a scene matching dataset similar to 3DMatch~\cite{zeng20173dmatch}, but having line and plane landmarks instead of point features.

Lines are detected in point clouds by first selecting pole-like points found by clustering the LiDAR range image, as in~\cite{dong2021online}.
For KITTI, pole-like points are selected using semantic information from SemanticKITTI~\cite{behley2019semantickitti}.
These points are clustered using DBSCAN~\cite{ester1996density} implemented in Open3D~\cite{zhou2018open3d} and PCA is used to estimate lines from pole-like clusters.
Planar patches are extracted from the LiDAR scan using our own implementation\footnote{\href{https://github.com/plusk01/pointcloud-plane-segmentation}{https://github.com/plusk01/pointcloud-plane-segmentation}} of~\cite{araujo2020robust}.
Because planar patches are bounded, there may be multiple planar patches that correspond to the same infinite plane, so we merge planes that are similar.
The statistics of correct line and plane landmark matches in the dataset are shown in Fig.~\ref{fig:dataset_details}, separated into three cases according to input inlier ratio (IIR), which is defined as the number of correct matches out of the number of possible matches and indicates the difficulty of each place pair.
Cases 1, 2, and 3 consist of place pairs with IIRs of $0.05+$, \SIrange{0.03}{0.05}{}, and \SIrange{0}{0.03}{}, respectively.
These IIR ranges are chosen to expose the relationship between IIR and sensor baseline distance and to ensure enough place pairs exist in each case.

\begin{figure}[t]
    \centering
    \vstretch{1}{\includegraphics[trim=0cm 0cm 0cm 0cm, clip, width=\columnwidth]{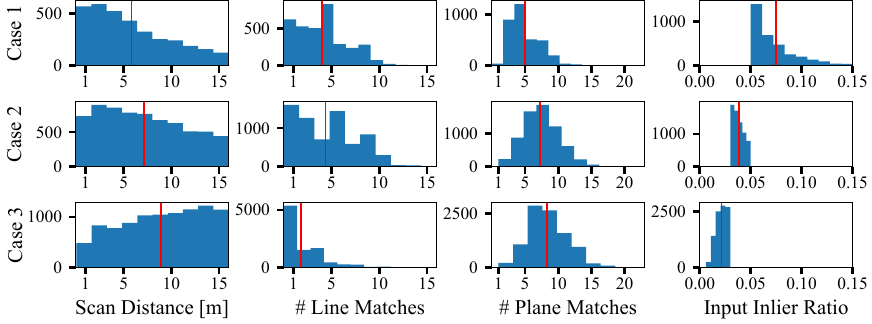}}
    \caption{
    Statistics of the 20k loop closures in the test dataset, divided into three cases according to each loop closure's input inlier ratio.
    The distance between sensor scans and input inlier ratio are correlated, with the most difficult case (Case 3) having longer sensor baselines.
    On average, there are more correct plane landmark matches than correct line landmark matches.
    }
    \label{fig:dataset_details}
\end{figure}

\subsection{Selection of Scaling Parameter}\label{sec:exp-scaling}

The scaling parameter $\rho$ (see Section~\ref{sec:consistency}) is chosen so that the pairwise affine Grassmannian distance lies in the linear regime and is therefore more sensitive when scoring consistencies.
The Velodyne HDL-64E used in KITTI has a range up to \SI{120}{\m}, with an average point range in the KITTI dataset of approximately \SI{80}{\meter}.
Additionally, the average Euclidean distance between landmark centroids in KITTI is \SI[separate-uncertainty=true,multi-part-units=single]{26\pm16}{\meter}, as shown in Fig.~\ref{fig:rho_sensitivity}.
Therefore, we select $\rho=40$ so that relative Euclidean distances of \SI{80}{\meter} will be scaled to \SI{2}{\meter}, which is at the end of the linear regime (see Fig.~\ref{fig:graffsensitivity}).
However, GraffMatch is not extremely sensitive to this choice and yields similar matching results for a range of $\rho$ values while holding other parameters constant, as shown in Fig.~\ref{fig:rho_sensitivity}.

\subsection{Evaluation Metrics}\label{sec:evalmetrics}
We evaluate both landmark matching ability and registration quality.
Correctness of landmark matching is evaluated via landmark-match recall (LMR)~\cite{choy2019fully}, which measures the percentage of place pairs that can be registered with high confidence given landmark matches.
LMR is defined as
\begin{equation}\small
\mathrm{LMR} = \frac{1}{N}\sum_{s=1}^N \mathds{1}\left( \left[ \frac{1}{|\mathcal{C}_s|} \sum_{(i,j)\in \mathcal{C}_s} \mathds{1}\left(d_{ij} < \tau_d\right)\right] > \tau_\mathrm{OIR} \right),
\end{equation}
where $N$ is the number of place pairs, $\mathcal{C}_s$ is a set of landmark correspondences between place pair $s$, and $d_{ij}$ is the $d_\mathrm{Graff}$ distance between landmark matches after registering landmarks using ground truth.
The inlier distance threshold $\tau_d$ controls how close two landmarks must be (after registering the landmarks using the ground truth) to be considered an inlier correspondence.
The output inlier ratio (OIR) measures the precision of an algorithm's correspondence set $\mathcal{C}_s$ and the threshold $\tau_\mathrm{OIR}$ is used to examine the percentage of place pairs an algorithm can recover with at least the specified OIR.
Point-based registration works have evaluated performance with an OIR threshold as low as 0.05~\cite{choy2019fully}, arguing that RANSAC can be effective even with only a 5\% inlier ratio, even though many iterations would be required. %
In contrast, we set \rev{$\tau_\mathrm{OIR}=0.8\gg0.05$} to emphasize the robustness of GraffMatch and we do not use RANSAC for inlier refinement.
\rev{The area under the LMR curve (AUC) generated by sweeping $\tau_\mathrm{OIR}\in[0,1]$ and holding $\tau_d=\SI{6}{\deg}$ is used as a summary measure, with an AUC of $1$ being ideal.}

Registration quality is evaluated via rotation, translation error and registration recall.
Error is calculated with respect to the ground truth transformation $(\MR^\star,\vt^\star)$ as $\arccos ((\mathrm{Tr}(\hat{\MR}^\top \MR^\star) - 1) / 2)$ and $\|\hat{\vt} - \vt^\star\|$.
Registration recall is the percentage of successfully registered place pairs (i.e., loop closures).
A successful registration has an estimation error within \SI{5}{\deg} and \SI{1}{\m} to the ground truth transformation.

\begin{figure}[t]
    \centering
    \begin{subfigure}[b]{0.49\columnwidth}
        \vstretch{1}{\includegraphics[trim=0cm 0cm 0cm 0cm, clip, width=1\columnwidth]{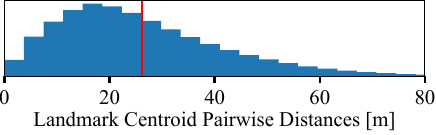}}
    \end{subfigure}
    \begin{subfigure}[b]{0.49\columnwidth}
        \vstretch{1}{\includegraphics[trim=0cm 0cm 0cm 0cm, clip, width=1\columnwidth]{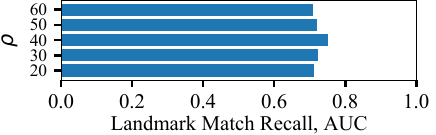}}
    \end{subfigure}
    \vskip-0.2em
    \caption{
    (left) Pairwise distances of landmark centroids in KITTI.
    The mean is \SI[separate-uncertainty=true,multi-part-units=single]{26\pm16}{\meter}.
    Using this data, we choose the scaling parameter as $\rho=40$.
    (right) GraffMatch is not extremely sensitive to this choice of scaling.
    Other values of $\rho$ yield similar matching results in KITTI, indicated by the AUC.
    }
    \label{fig:rho_sensitivity}
\end{figure}

\begin{figure*}[t]
    \centering
    \vstretch{0.94}{\includegraphics[trim=0cm 0cm 0cm 0cm, clip, width=\textwidth]{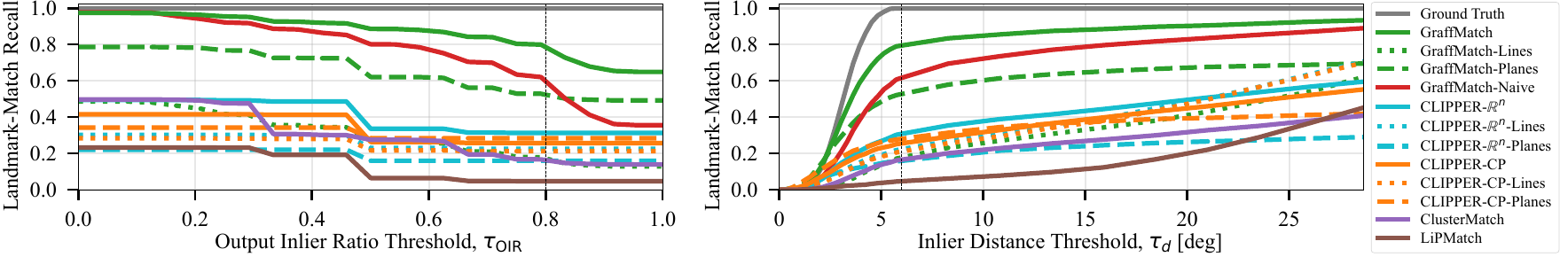}}
    \caption{
    (left) Landmark-match recall (LMR) varied over $\tau_\mathrm{OIR}$, indicating for each algorithm the percentage of place pairs having selected correspondences with at least the given OIR.
    (right) The maximum distance a pair of ground-truth registered landmarks can be while considered an inlier match.
    }
    \label{fig:lmr}
\end{figure*}

\begin{table*}[!h]
\scriptsize
\centering
\caption{
Results for each dataset, divided into three cases based on the feature inlier ratio.
We compare each algorithm on recall rate and landmark-match recall AUC and the highest for each dataset case is bolded.
The average rotation and translation errors of the successful registrations are also listed.
}
\setlength{\tabcolsep}{1.7pt}
\ra{1.05}
\begin{tabular}{p{0.75em} l l c c c c c c c c c c c c c c c c c c c c c c c c c c c c c c c c c c c c}
\toprule
&
&& \multicolumn{4}{c}{CPMatch~\cite{zhou2021pi}}
&& \multicolumn{4}{c}{BruteForceRMSE~\cite{zhou2022plc}}
&& \multicolumn{4}{c}{ClusterMatch~\cite{cao2020accurate}}
&& \multicolumn{4}{c}{LiPMatch~\cite{jiang2020lipmatch}}
&& \multicolumn{4}{c}{CLIPPER-$\mathbb{R}^n$}
&& \multicolumn{4}{c}{CLIPPER-CP}
&& \multicolumn{4}{c}{GraffMatch}\\
\cmidrule{4-7}\cmidrule{9-12}\cmidrule{14-17}\cmidrule{19-22}\cmidrule{24-27}\cmidrule{29-32}\cmidrule{34-37}
&& $N$
 & \recall & \auc & \rerr & \terr
&& \recall & \auc & \rerr & \terr
&& \recall & \auc & \rerr & \terr
&& \recall & \auc & \rerr & \terr
&& \recall & \auc & \rerr & \terr
&& \recall & \auc & \rerr & \terr
&& \recall & \auc & \rerr & \terr\\[0.5em]
\toprule
\multirow{3}{*}{\rotatebox[origin=c]{90}{\tiny KITTI}}
& Case 1 & $1392$
 & $24$ & $0.47$ & $1.1$ & $21$ && $9$  & $0.33$ & $1.2$ & $62$ && $35$ & $0.34$ & $1.7$ & $18$
&& $16$ & $0.30$ & $1.1$ & $19$ && $62$ & $0.70$ & $1.5$ & $24$ && $42$ & $0.56$ & $1.3$ & $21$
&& $\mathbf{81}$ & $\mathbf{0.91}$ & $1.1$ & $20$ \\
& Case 2 & $4560$
 & $19$ & $0.53$ & $1.2$ & $23$ && $3$  & $0.34$ & $1.3$ & $69$ && $16$ & $0.18$ & $1.5$ & $17$
&& $15$ & $0.32$ & $1.2$ & $22$ && $46$ & $0.55$ & $1.5$ & $27$ && $22$ & $0.38$ & $1.4$ & $25$
&& $\mathbf{52}$ & $\mathbf{0.78}$ & $1.1$ & $19$ \\
& Case 3 & $2329$
 & $7$  & $0.42$ & $1.5$ & $29$ && <$1$ & $0.23$ & $2.1$ & $66$ && $6$  & $0.10$ & $1.5$ & $21$
&& $6$  & $0.23$ & $1.7$ & $30$ && $20$ & $0.30$ & $1.7$ & $31$ && $7$  & $0.17$ & $1.7$ & $33$
&& $\mathbf{21}$ & $\mathbf{0.59}$ & $1.2$ & $24$ \\
\midrule
\multirow{3}{*}{\rotatebox[origin=c]{90}{\tiny KITTI-360}}
& Case 1 & $182$
 & $17$ & $0.51$ & $1.6$ & $29$ && $3$ & $0.34$ & $1.5$ & $28$ && $21$ & $0.28$ & $1.4$ & $26$
&& $8$  & $0.23$ & $1.8$ & $35$ && $16$ & $0.32$ & $1.5$ & $26$ && $10$ & $0.35$ & $1.7$ & $32$
&& $\mathbf{53}$ & $\mathbf{0.85}$ & $1.6$ & $32$ \\
& Case 2 & $1555$
 & $10$ & $0.43$ & $1.6$ & $27$ && $1$ & $0.26$ & $2.4$ & $20$ && $8$ & $0.14$ & $1.3$ & $25$
&& $5$  & $0.19$ & $1.7$ & $32$ && $12$ & $0.26$ & $2.0$ & $31$ && $6$ & $0.28$ & $1.7$ & $30$
&& $\mathbf{41}$ & $\mathbf{0.74}$ & $1.8$ & $36$ \\
& Case 3 & $7143$
 & $4$  & $0.27$ & $1.8$ & $28$ && <$1$ & $0.14$ & $2.6$ & $20$ && $2$ & $0.06$ & $1.7$ & $31$
&& $2$  & $0.11$ & $2.0$ & $36$ && $6$ & $0.14$ & $2.2$ & $36$ && $1$ & $0.14$ & $2.1$ & $34$
&& $\mathbf{21}$ & $\mathbf{0.55}$ & $1.9$ & $42$ \\
\midrule
\multirow{3}{*}{\rotatebox[origin=c]{90}{\tiny NCLT}}
& Case 1 & $1854$
 & $6$  & $0.33$ & $2.0$ & $30$ && $1$ & $0.23$ & $2.6$ & $81$ && $21$ & $0.36$ & $2.9$ & $30$
&& $4$  & $0.16$ & $2.5$ & $41$ && $25$ & $0.42$ & $2.7$ & $35$ && $14$ & $0.39$ & $2.7$ & $35$
&& $\mathbf{60}$ & $\mathbf{0.91}$ & $2.7$ & $35$ \\
& Case 2 & $676$
 & $5$  & $0.38$ & $2.0$ & $25$ && $1$ & $0.20$ & $2.8$ & $89$ && $7$ & $0.17$ & $2.9$ & $34$
&& $1$  & $0.09$ & $1.9$ & $44$ && $20$ & $0.36$ & $2.9$ & $39$ && $9$ & $0.32$ & $2.8$ & $36$
&& $\mathbf{43}$ & $\mathbf{0.78}$ & $2.8$ & $38$ \\
& Case 3 & $150$
 & $4$  & $0.32$ & $2.5$ & $29$ && <$1$ & $0.14$ & $2.6$ & $76$ && $2$ & $0.15$ & $3.2$ & $19$
&& $0$  & $0.07$ &  --   &  --  && $12$ & $0.31$ & $3.2$ & $43$ && $2$ & $0.23$ & $2.4$ & $42$
&& $\mathbf{16}$ & $\mathbf{0.62}$ & $2.9$ & $36$ \\
\bottomrule
\end{tabular}
\label{tbl:results}
\vskip-0.2in
\end{table*}

\subsection{Landmark Matching and Registration Results}
Data association is attempted on each place pair, after which the landmark matches are used to estimate the relative rotation and translation.
Fig.~\ref{fig:lmr} plots algorithms' LMR curves for the NCLT dataset.
GraffMatch is able to match \rev{approximately $80\%$ of place pairs with an OIR of $0.8$ or more}, while other methods are only able to match less than $40\%$ of place pairs for the same inlier regime.
For correspondence sets having such a high inlier ratio, additional methods like RANSAC could be used to further refine the correspondence set if desired.
However, we show that even without this extra step, GraffMatch is able to successfully select the most correct correspondences.
Fig.~\ref{fig:lmr} also includes line-only and plane-only variants of GraffMatch to highlight the value of creating a consistency graph utilizing both landmark types.
In the NCLT dataset, there are fewer detected poles than planes, leading to a gap between the line-only and plane-only variants; however, using both landmarks results in a greater number of place pairs having a higher OIR.
\rev{Additionally, a version denoted GraffMatch-Naive is included that does not first shift landmark pairs (as in Proposition~\ref{prop:invariance}), and so does not use a distance that is transformation invariant.
The gap between GraffMatch and GraffMatch-Naive highlights the importance of using an invariant distance function when scoring landmark association consistency.}
In the right of Fig.~\ref{fig:lmr}, we see how the LMR curve changes with the inlier distance threshold $\tau_d$.
By using $d_\mathrm{Graff}$ for these geometric landmarks, GraffMatch is able to achieve an LMR curve most similar to the ground truth.
Due to space limitations, the LMR curves for each dataset are not shown, but the AUC for each dataset and case is listed in Table~\ref{tbl:results}.
In all datasets and cases, GraffMatch successfully matches more landmarks and has higher registration recall than other methods.
As expected, the local methods CPMatch and BruteForceRMSE perform worse than global methods because of wide sensor baselines present in the datasets.
However, with the exception of GraffMatch, the global methods leverage landmark properties that are view dependent.
ClusterMatch and CLIPPER-$\mathbb{R}^n$ use landmark centroids, which are not well-defined for infinite geometries and may shift depending on the detection.
Likewise, LiPMatch uses centroids as well as properties such as the area and the extent of planar patches.
CLIPPER-CP uses CP vectors, which are not invariant to translation and cannot be shifted as in Proposition~\ref{prop:invariance} because the CP vector is undefined for landmarks at the origin.
These results underscore the importance of using view-independent geometric representations for lines and planes, e.g., the affine Grassmannian mainfold used in the GraffMatch framework.

\begin{figure*}[t]
    \centering
    \vstretch{0.95}{\includegraphics[trim=0cm 0cm 0cm 0cm, clip, width=\textwidth]{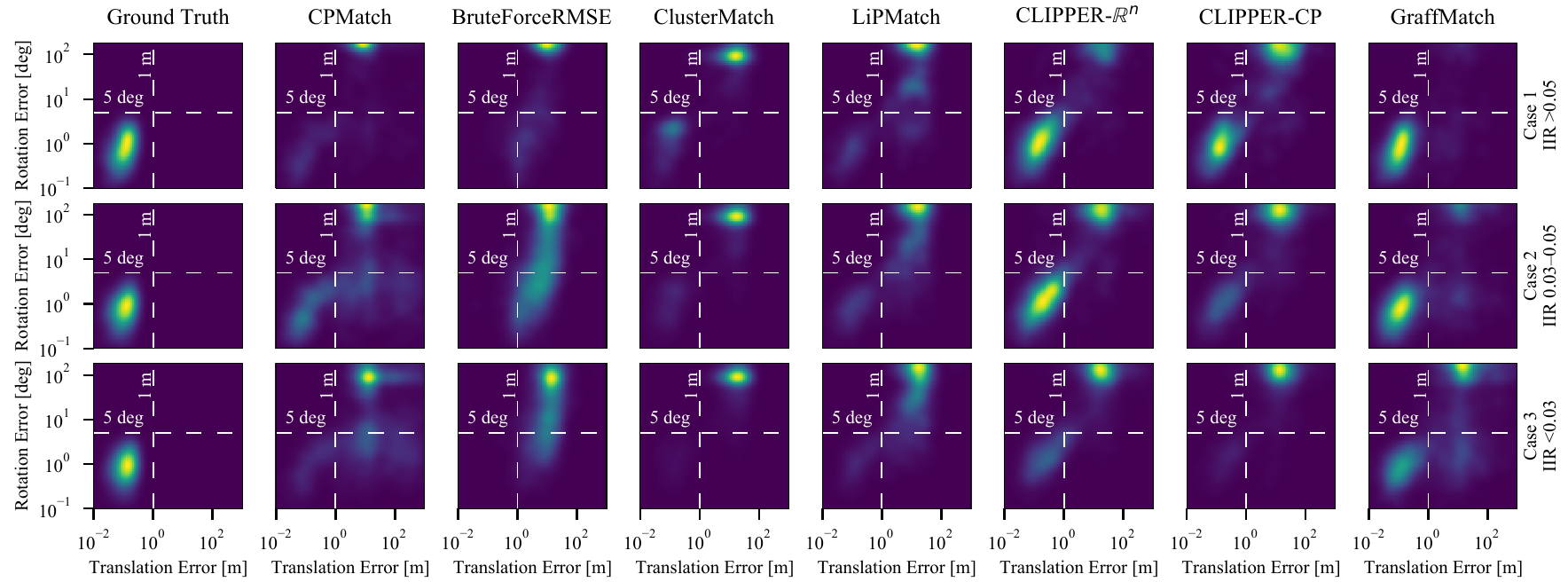}}
    \caption{
    Alignment error for place pairs, visualized as likelihood-normalized density plots.
    Columns correspond to algorithms.
    The first column shows the best alignment error achievable when registering landmarks using the true landmark matches.
    From top to bottom, each row corresponds to Cases 1, 2, and 3, with Case 1 being easiest and Case 3 being hardest.
    In each case, GraffMatch produces the greatest number of registrations with low alignment error. %
    }
    \vskip-0.2in
    \label{fig:error_plots}
\end{figure*}

Fig.~\ref{fig:error_plots} shows the alignment error of all attempted place pair registrations from KITTI as a grid of density heatmaps, where columns correspond to algorithms and rows (from top to bottom) correspond to Cases 1, 2, and 3.
The first column shows the alignment error using ground truth landmark matches with respect to ground truth alignment, indicating the best achievable landmark-based registration without a point-based refinement step (e.g., using ICP).
GraffMatch is the only data association method that consistently scores in the low-translation, low-rotation error regime.

Runtime statistics are shown in Fig.~\ref{fig:runtimes}.
Compared to CLIPPER-$\mathbb{R}^n$, GraffMatch incurs an additional \SI{50}{\milli\s} on average due to the calculation of $d_\mathrm{Graff}$, but is still capable of real-time at \SI{10}{\Hz} LiDAR rate.
GraffMatch runtime could be reduced by encoding prior knowledge in the putative associations, e.g., ground planes should be matched, or large planar patches are not likely to be matched to small patches.

\rev{
Although storing and processing raw points can be resource intensive, directly operating on point clouds is a typical means of registration and loop closure constraint generation.
Thus, we compare GraffMatch with existing point-based methods, with results in Table~\ref{tbl:vs-pts}.
As a baseline, we use FPFH~\cite{rusu2009fast} local descriptors (downsampled to \SI{0.35}{\m}) with RANSAC (max \num[group-separator={,}]{100000} iters.) implemented in Open3D.
Scan context~\cite{kim2021scan}, which is nominally used for place recognition, is also included.
Different from other point cloud place recognition methods, it provides an initial guess for yaw that can be used to initialize an ICP-based registration step.
For a fair comparison, Table~\ref{tbl:vs-pts} also reports GraffMatch solutions refined by ICP, but only using points of matched landmarks, thereby reducing storage and processing requirements.
Because scan context is only rotation invariant, it achieves relatively low recall since ground truth place pairs can have up to \SI{16}{\m} of translation.
Performing ICP increases the success rate, but only marginally when compared to GraffMatch without ICP refinement (see Table~\ref{tbl:results}).
These results on this dataset indicate the translation sensitivity of scan context, which limits it to small-baseline settings.
Similarly performing ICP on GraffMatch solutions leads to high success and accuracy, indicating that the landmark matches are of high quality (i.e., high LMR AUC in Table~\ref{tbl:results}), but there may be cases where the landmark-only least squares registration was ill-conditioned.
The FPFH baseline performs poorly due to non-uniform LiDAR sampling and low point cloud overlap, and its high computational cost highlights the potential difficulties of processing point clouds directly.}

\begin{figure}[t]
    \centering
    \vstretch{1}{\includegraphics[trim=0cm 0cm 0cm 0cm, clip, width=\columnwidth]{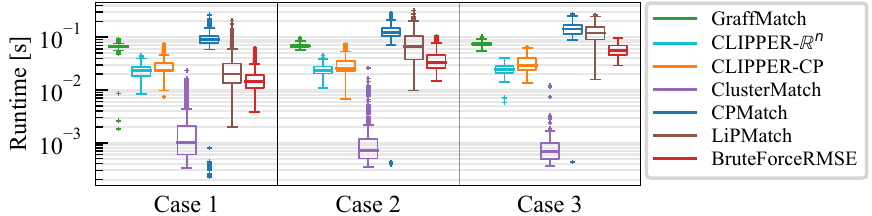}}
    \caption{
    Runtime statistics.
    GraffMatch is capable of running in real-time at the typical 10 Hz LiDAR rate.
    The runtime of GraffMatch could be reduced using landmark descriptors to decrease the number of putative associations.
    }
    \vskip0.1in
    \label{fig:runtimes}
\end{figure}

\begin{table}[t]
\scriptsize
\centering
\caption{
\rev{
Point-based registration methods are computationally demanding and do not necessarily lead to higher accuracy or success in wide-baseline settings.
Refining GraffMatch with ICP (only using points underlying landmarks) leverages the positive aspects of landmark and point-based registration.}
}
\setlength{\tabcolsep}{1pt}
\ra{1.05}
\begin{tabular}{p{0.75em} l c c c c c c c c c c c c c c c c c c c}
\toprule
&& \multicolumn{4}{c}{GraffMatch+ICP}
&& \multicolumn{4}{c}{FPFH+RANSAC}
&& \multicolumn{4}{c}{ScanContext~\cite{kim2021scan}}
&& \multicolumn{4}{c}{ScanContext+ICP}\\
\cmidrule{3-6}\cmidrule{8-11}\cmidrule{13-16}\cmidrule{18-21}
&& \recall & \rerr & \terr & \time
&& \recall & \rerr & \terr & \time
&& \recall & \rerr & \terr & \time
&& \recall & \rerr & \terr & \time\\[0.5em]
\toprule
\multirow{3}{*}{\rotatebox[origin=c]{90}{\tiny KITTI}}
& Case 1
 & $\mathbf{93}$ & $0.5$ & $13$ & $119$
&& $16$ & $1.7$ & $46$ & $1029$
&& $8$ & $0.9$ & $54$ & $25$
&& $29$ & $0.5$ & $21$ & $142$ \\
& Case 2
 & $\mathbf{76}$ & $0.4$ & $7$ & $125$
&& $12$ & $1.8$ & $44$ & $807$
&& $4$ & $0.9$ & $59$ & $25$
&& $31$ & $0.5$ & $25$ & $138$ \\
& Case 3
 & $\mathbf{49}$ & $0.3$ & $8$ & $156$
&& $5$ & $2.2$ & $51$ & $660$
&& <$1$ & $1.0$ & $52$ & $37$
&& $14$ & $0.6$ & $32$ & $140$ \\
\bottomrule
\end{tabular}
\label{tbl:vs-pts}
\end{table}

\subsection{Automatic LiDAR-LiDAR and Camera-Depth Calibration}
High-quality extrinsic calibration is a crucial prerequisite for many autonomous systems.
\rev{Calibration methods typically require data correspondence across sensors and state-of-the-art multimodal calibration often achieve this by requiring fiducial markers on calibration targets~\cite{beltran2022automatic}.}
\rev{Instead, unlabeled geometric landmarks extracted from each modality can be matched and then registered using GraffMatch.}
Fig.~\ref{fig:calib} shows two instances of extrinsic calibration using GraffMatch to globally match planes.
In Fig~\ref{fig:calib-shuttle}, many chess boards are held in the field-of-view of the LiDARs and planes are extracted from the two point clouds (red and blue).
GraffMatch correctly matches $10$ planes in \SI{75}{\milli\second} with calibration error of \SI{1.1}{\deg} and \SI{8}{\centi\meter}.
In Fig~\ref{fig:calib-d435i}, an example of multi-modal calibration using is presented, where 3D planes are extracted from chess boards in the image frame and matched to planes extracted from the depth sensor.
GraffMatch correctly matches $7$ planes in \SI{23}{\milli\second} resulting in a calibration error of \SI{0.9}{\deg} and \SI{5}{\centi\meter}.

\begin{figure}[t]
    \centering
    \begin{subfigure}[b]{1\columnwidth}
        \vstretch{0.9}{\includeinkscape[pretex=\tiny,width=\columnwidth]{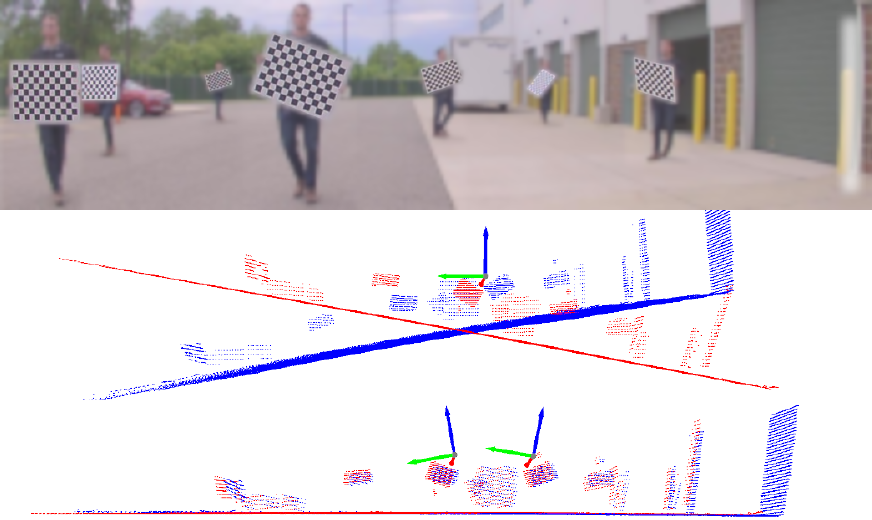}}
        \caption{}
        \label{fig:calib-shuttle}
    \end{subfigure}
    \vskip0.5em
    \begin{subfigure}[b]{1\columnwidth}
        \vstretch{0.9}{\includeinkscape[pretex=\tiny,width=\columnwidth]{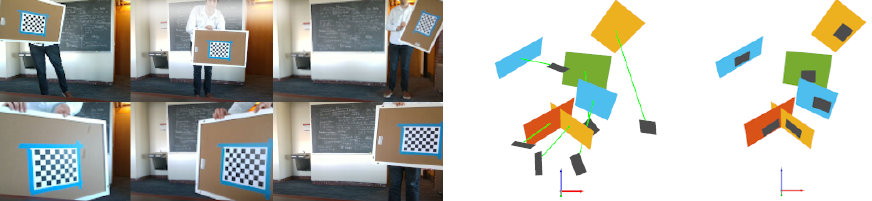}}
        \caption{}
        \label{fig:calib-d435i}
    \end{subfigure}
    \caption{
    Automatic extrinsic calibration using GraffMatch.
    (a) 3D planes are extracted from two Ouster LiDAR scans, illustrated with red and blue point clouds.
    Plane-to-plane correspondences are identified, allowing the sensors to be calibrated without an initial guess.
    (b) 3D plane detections of calibration targets are extracted from depth sensor and from intrinsically-calibrated RGB sensor using PnP.
    RGB and depth data collected from Intel D435i are arbitrarily transformed to simulate uncalibrated sensors.
    }
    \label{fig:calib}
\end{figure}

%% file: paper/conclusion.tex
\section{Conclusion}\label{sec:conclusion}

We presented GraffMatch, a global method for matching and registering 3D lines and planes from two landmark sets without an initial alignment guess.
Our main contribution is the representation and data association of lines and planes as elements of the affine Grassmannian manifold.
By naturally representing these geometric landmarks as affine subspaces, we leverage the Grassmannian metric to calculate the distance between two landmarks.
We prove that affine Grassmannian distance is invariant to rotation and translation provided a shift operation is first applied.
This invariance property enables the use of efficient and robust graph-theoretic data association.

Future research includes reducing runtime and avoiding symmetries by developing landmark descriptors for generating putative matches (i.e., instead of an all-to-all hypothesis), and
estimating lines and planes directly via subspace tracking methods and using manifold-based optimization techniques to perform online bundle adjustment of affine Grassmannian landmarks within a SLAM framework.

%% file: paper/appendix.tex
\appendix

\subsection{Proof of Invariance}

\begin{proof}
\rev{
Suppose $\bY_1,\bY_2$ are shifted by $-b_1$ such that $b_{1}=0$.
This implies that the orthogonal displacement of $\bY_1$ is $b_{01}=0$ and the inner product of the Stiefel coordinates of $\bY_1,\bY_2$ is
\begin{equation*}
Y_1^\top Y_2 =
\begin{bsmallmatrix}
A_1^\top A_2 & \frac{1}{\eta_2}A_1^\top b_{02} \\
\frac{1}{\eta_2}b_{01}^\top A_2 & \frac{1}{\eta_1\eta_2}\left(b_{01}^\top b_{02} + 1\right)
\end{bsmallmatrix}
=
\begin{bsmallmatrix}
A_1^\top A_2 & \frac{1}{\eta_2}A_1^\top b_{02} \\
0 & \frac{1}{\eta_1\eta_2}
\end{bsmallmatrix},
\end{equation*}
with $\eta_i\eqdef\sqrt{\|b_{0i}\|^2 + 1}$.
Recall that computing $d_\Graff(\bY_1,\bY_2)$ uses the SVD of $Y_1^\top Y_2$.
Given $T=(R,t)\in\mathrm{SE}(3)$, let $\bar{\bY}_1,\bar{\bY}_2$ be the transformations of $\bY_1,\bY_2$,  with affine coordinates
\begin{equation*}
\bY_i:[A_i,b_i] \xrightarrow{\quad T\quad} \bar{\bY}_i:[RA_i, Rb_i + t].
\end{equation*}
Shifting $\bar{\bY}_1,\bar{\bY}_2$ by $-\bar{b}_1\eqdef-(Rb_1+t)$ leads to the affine coordinates $\bar{\bY}_1:[RA_1, 0]$ and $\bar{\bY}_2:[RA_2, Rb_2]$ so that
\begin{equation*}
\bar{Y}_1^\top \bar{Y}_2 =
\begin{bsmallmatrix}
A_1^\top A_2 & \tfrac{1}{\eta_2}A_1^\top b_{02} \\
0 & \tfrac{1}{\eta_1\eta_2}
\end{bsmallmatrix}
= Y_1^\top Y2,
\end{equation*}
implying that $d_\Graff(\bY_1,\bY_2)$ is invariant to $(R,t)$.}
\end{proof}